\documentclass[12pt]{article}

\usepackage{amsmath}
\usepackage{amssymb}
\usepackage{amsthm}
\usepackage{subfigure}
\usepackage{multirow}
\usepackage{hyperref}
\usepackage[numbers,sort&compress]{natbib}

\usepackage{graphicx}
\usepackage{float}

\newcommand{\bs}[1]{\boldsymbol{#1}}
\newcommand{\dive}{\operatorname{div}}
\newtheorem{prop}{Proposition}
\newcommand{\mc}[1]{\mathcal{#1}}
\newcommand{\Fb}[1]{\boldsymbol{\mathfrak{#1}}}
\newcommand{\Mb}[1]{\mathbb{#1}}
\renewcommand{\vec}[1]{\ensuremath{\boldsymbol{#1}}}
\begin{document}

\title{Pre-process for segmentation task with nonlinear diffusion filters}

\author{Javier Sanguino \and Carlos Platero \and Olga Velasco\\[6pt]
\normalsize Health Science Technology Group, Technical University of Madrid,\\
\normalsize Ronda de Valencia 3, 28012, Madrid, Spain.}

\date{Manuscript originally written in January 2017. Submitted to arXiv, 2026.}

\maketitle

\begin{abstract}
This paper deals with the case of using nonlinear diffusion filters to obtain
piecewise constant images as a previous process for segmentation techniques.

We first show an intrinsic formulation for the nonlinear diffusion equation to
provide some design conditions on the diffusion filters. According to this
theoretical framework, we propose a new family of diffusivities; they are
obtained from nonlinear diffusion techniques and are related with backward
diffusion. Their goal is to split the image in closed contours with a
homogenized grey intensity inside and with no blurred edges.

We also prove that our filters satisfy the well-posedness semi-discrete and
full discrete scale-space requirements. This shows that by using semi-implicit
schemes, a forward nonlinear diffusion equation is solved, instead of a
backward nonlinear diffusion equation, connecting with an edge-preserving
process. Under the conditions established for the diffusivity and using a
stopping criterion for the diffusion time, we get piecewise constant images
with a low computational effort.

Finally, we test our filter with real images and we illustrate the effects of
our diffusivity function as a method to get piecewise constant images.

The code is available at \url{https://github.com/cplatero/NonlinearDiffusion}.
\end{abstract}

\noindent\textbf{MSC classes:} 68U10 (Image processing), 68T45 (Machine vision and scene understanding), 65M06 (Finite difference methods)

\section{Introduction}
\label{intro}
A typical application of the Computer Vision is the  segmentation process, where initial images show noise, regions without homogeneity in their intensities, weak edges and small artifacts inside of regions of interest.

However,  segmentation techniques  are usually based on the fact that
the regions are piecewise  constant~\cite{RefWorks:47} and the edges have
high slopes. Hence 
segmentation approaches need that the images have to be filtered previously  in
order to homogenize the regions and enhance the edges.

Therefore, given an observed image $u_{0}$ we are interested in finding another image $u$, \emph{close} to $u_{0}$, formed by homogeneous regions and with sharp boundaries. The \emph{piecewise constant images} or \emph{cartoon images}, that we are interested to obtain, should be composed by regions delimited by closed and sharped contours with a constant grey value inside.

The cartoon images arise in some models in image analysis where a relation is assumed
between an observed image $u_{0}$ and the cartoon component $u$ as follows: $u_{0} = u + v$ where $v$ is noise or small scale repeated detail (texture). In this case, both types of patterns (additive noise or texture) can be modeled by oscillatory functions taking both positive and negative values and zero mean~\cite{RefWorks:169}.  These models yield to a total variation minimization problem where $u\in BV$ is a function of Bounded Variation, while $v$ belongs to $G$, a space of oscillating functions which is the dual of the closure in $BV(\mathbb{R}^{2})$ in the Schwartz class~\cite{RefWorks:169}.
In some contexts the $v$ component is important, especially if we are interested in image texture~\cite{RefWorks:169,RefWorks:170,RefWorks:163,RefWorks:165,RefWorks:164}.
On the other hand, some models in a context of image recovery or classical image denoising suggest that the $v$ component models the noise in a total variation minimization framework~\cite{RefWorks:129,RefWorks:36,RefWorks:173,RefWorks:177}.

For image segmentation, the Mumford-Shah model \cite{RefWorks:128} is one of the first image decompositions in texture and cartoon components.
Furthermore, it is also known the connection between the Mumford and Shah model and the Perona-Malik evolution equations~\cite{RefWorks:30} as it is shown in~\cite{RefWorks:132}. We use this fact as a motivation to get a nonlinear diffusion framework in order to obtain a kind of cartoon image
from the initial image.
Thus, in this work we are interested in extracting the cartoon component $u$ from $u_{0}$ eliminating the component $v$ that contains noise and some texture of the initial image.
Our approach is not based on a total variation minimization model, where the functional has two terms: one measure the fidelity to the data and the other is a smoothing term~\cite{RefWorks:129,RefWorks:36,RefWorks:11}.
We only consider a smoothing term, with a nonconvex functional of type
\begin{equation}\label{eq:funcional}
\Psi(u)  = \int_{\Omega} \rho\left(\|\nabla u\|\right)\, d\vec{x}
\end{equation}
We will study, for continuous images the model described by Perona and Malik~\cite{RefWorks:30}. Thus, we begin from the definition of \emph{scalar potential diffusivity function}
\begin{equation}\label{eq:intri_7}
\rho(s) = \frac{1}{2}\int_{0}^{s^{2}} g(t)\, dt
\end{equation}
where $g: \mathbb{R}^{+}\cup \{ 0\} \to \mathbb{R}^{+}$ is initially a
smooth enough positive function called \emph{diffusivity}. Hence we consider the functional given by~\eqref{eq:funcional} and  we minimize it by using the descent gradient method to get the following Euler-Lagrange equation
\begin{equation}\label{eq:intro_1}
\left\{\begin{array}{cll}
\partial_{t} u = \dive\, \left(g(\|\nabla u\|^{2})\, \nabla u\right) & \text{if} \quad (\vec{x},t)\in \Omega \times (0,T)&\\[5pt]
u(\vec{x},0) = u_{0}(\vec{x}) & \text{if} \quad \vec{x} \in \Omega\;&\\[5pt]
\partial_{\vec{n}}u = 0 & \text{if}
\quad (\vec{x},t) \in \partial \Omega \times (0,T]&
\end{array}\right.
\end{equation}
Given an initial image $u_{0}:\Omega
\rightarrow\mathbb{R}$ defined over a bounded domain $\Omega \subset
\mathbb{R}^{s}$ (with $s =1,2$ or $3$) we get an image $u(\vec{x},T)$
as the solution to a nonlinear diffusion equation with initial and
Neumann boundary conditions.

The aim of this paper is to propose  a diffusivity function with an edge-preserving behaviour that  allows us to implement a criterion for establishing a stopping time of the diffusion process. We present a
new and a robust procedure to split  the original image into  piecewise constant regions. In this regard, Black~\emph{et al}~\cite{RefWorks:157} established that nonlinear diffusion can
be seen as a robust estimation procedure getting a piecewise constant image from an observed image. So, our proposal is to present a diffusivity function in order to obtain piecewise constant images using nonlinear diffusion techniques.

It is clear that the piecewise constant images should be related with the edge enhancement or at least an edge-preserving process, so we first propose a theoretical framework for getting these kind of diffusivities.

In this regards, these behaviors are directly related to the non convexity of $\rho$ given by~\eqref{eq:intri_7}~\cite{RefWorks:11,RefWorks:36}. In this case, the problem is ill-posed and might have no solution, and one cannot prove any convergence result. One way
to transform it to a \emph{well-posed} problem is proposed by Catte
\emph{et al}~\cite{RefWorks:51}. In this case the image $u_{0}$ is
convolved with a Gaussian mask of standard deviation $\sigma$ to obtain
$u_{\sigma}$ so they get a pre-smoothing mechanism. This pre-smoothing mechanism can be a drawback in an edge enhancement process.
Nevertheless, discretizations may have a regularizing effect to
Perona-Malik `ill-posed' problem as it is shown by
Weickert~\cite{RefWorks:31}. Using the \emph{Method of Lines} (MOL),
we arrive to a semi-discretization of the problem~\eqref{eq:intro_1}
and we show that our proposed diffusivity satisfies the criterion for
semi-discrete diffusion scale-spaces framework~\cite{RefWorks:130}; this is
sufficient to transform the former partial differential
equation~\eqref{eq:intro_1} in a well-posed system of ordinary
differential equations~\cite{RefWorks:53}.
However, these discrete and iterative solutions tend to the average  grey value of the image. Therefore, it is necessary that we introduce a criterion to stop the diffusion process to facilitate the access to the appropriate images.

Forward-and-Backward (FAB) diffusion model proposed by Gilboa et al \cite{RefWorks:37} give us another research area \cite{wang2007local,welk2009theoretical,Wang} where its diffusivities are positive in certain regions and negative in others in order to combine smoothing and sharpening actions. As a consequence, it seeks an enhancement process instead of an edge-preserving behaviour. This  model has some difficulties to satisfy the well-posedness semi-discrete and full discrete scale-space requirements \cite{welk2009theoretical}. For these reasons,  in this work, we have chosen to use an edge-preserving process where we have the possibility of using semi-implicit time discretizations that allows us to use a larger time steps with the idea of reducing the computational effort.

The outline of the paper is as follows. First, based upon the Perona and Malik model~\eqref{eq:intro_1},  we
present an intrinsic formulation that can be used to set the design conditions for edge enhancement diffusivity
functions for 2 and 3 dimension, in order to get a theoretical framework that we will
connect with a numerical scheme related with an edge-preserving process.
It is also shown that the Perona-Malik function~\cite{RefWorks:30} and some others diffusivities
are  particular cases of these design conditions.
In section~$3$, we propose a diffusivity function with the aim to obtain a piecewise constant image from an observed image and with the feature of providing a stopping time for the diffusion process.
In section~\ref{sec:numerical1D},  we show  that our diffusivity achieves
the necessary conditions for the existence and uniqueness of
solutions using the scale-space framework established by Weickert
and Benhamouda~\cite{RefWorks:31}. For simplicity, we begin applying the Method of Lines in 1D.
Then we  proceed with time discretization. We study the implicit
scheme from the Backward Euler method and how it leads to a nonlinear
algebraic problem. The important point to note here is the
singularity of the iteration matrix.  We analyze the Jacobian
of the system of the differential equation: it is a function of two matrices, one of them containing the effects of the backward
diffusion. To avoid the singularities of the iteration matrix, we propose to use the \emph{tangential method}~\cite{RefWorks:92,RefWorks:93} which
leads to \emph{Picard iteration} and we show the condition for
convergence. We realize that the Picard scheme agrees with the semi-implicit
scheme for one dimension. It also allows to point out that in fact we
do not do backward diffusion but forward diffusion with a positive decreasing diffusivity function that slows down the diffusion connecting with an edge-preserving process. In Section~\ref{sec:numerical2D} we extend
straightforwardly the results obtained in 1D to 2D and 3D; this leads to the \emph{Additive Operator Splitting} (AOS) method~\cite{RefWorks:42,RefWorks:130}.
Finally, in the last section we validate our diffusivity
with a semi-implicit scheme for real images in 2D and abdominal Computer Tomography (\emph{CT}) scans.

\section{Condition for edge enhancing in a nonlinear diffusion process}
\label{sec:1}
In this section,
we establish the condition that the diffusivity $g$ should satisfy in order to get an edge enhancement process with equation~\eqref{eq:intro_1}.

\subsection{Intrinsic formulation for nonlinear diffusivity}
\label{sec:2}
According to the expression of $\rho$ given in~\eqref{eq:intri_7} we get: $\rho'(s) = g(s^{2})\, s$. Then we
define the \emph{diffusion flow} as:
\begin{equation}\label{eq:intri_1}
\bs{F} = g(\|\nabla u\|^{2}) \, \nabla u
\end{equation}
and $\bs{\eta}=\nabla\, u/\|\nabla u\|$ with $\|\nabla u\|\neq 0$. Using $\bs{\eta}$ and adding orthogonal and unitary vectors we make positive oriented systems $\widehat{B} =
\{\bs{\xi},\bs{\eta}\}$ in $\mathbb{R}^{2}$  and $\widehat{B}
=\{\bs{\xi},\bs{\zeta},\bs{\eta} \}$ for
$\mathbb{R}^{3}$.

Now, if $\bs{I}$ is the identity tensor, `$Tr$' is the Trace operator
and `$\bs{:}$' is the tensor scalar product ($\mathbf{A}\boldsymbol{:} \mathbf{B} = Trace(\mathbf{A}^{t}\cdot \mathbf{B})$, see Gurtin~\cite{RefWorks:59}) we can write
\begin{equation*}
\begin{split}
Tr(\nabla \bs{F}) &=\bs{I} \bs{:}  \nabla \bs{F} =  \dive\, \left(g(\|\nabla u\|^{2})\, \nabla u\right)\\
&=\underbrace{\nabla
g(\|\nabla u\|^{2})\bs{\cdot} \nabla u}_{\text{Transport}}
+\underbrace{g(\|\nabla u\|^{2})\,\Delta u}_{\text{Diffusion}}
\end{split}
\end{equation*}
However, we are interested
in a intrinsic expression that shows, how the diffusivity $g$ affects the differential operator in equation~\eqref{eq:intro_1}. This is done as follows:
\begin{prop}\label{prop:intrinsic}
If $u$ and $g$ are smooth enough then
\begin{equation}\label{eq:intri_11}
\dive\, \left(g(\|\nabla u\|^{2})\, \nabla u\right) =
\Big[g(\|\nabla u\|^{2})\bs{I} + 2\,g'(\|\nabla u\|^{2}) \,\nabla u
\otimes \nabla u\Big]\bs{:}\nabla(\nabla u)
\end{equation}
where the size of identity tensor depends on space dimension.
\end{prop}
\begin{proof}
We will use some basic tensor properties~\cite{RefWorks:59}.
From~\eqref{eq:intri_1}, we get
\begin{equation*}
\nabla \bs{F} =\nabla u\otimes \nabla g(\|\nabla u\|^{2})+g(\|\nabla
u\|^{2})\,\nabla(\nabla u)
\end{equation*}
Using the properties of the \emph{Trace}
operator, we get
\begin{equation*}
Tr (\nabla \bs{F}) = Tr(\,\nabla u\otimes \nabla g(\|\nabla
u\|^{2})\,) + Tr(\,g(\|\nabla u\|^{2})\,\nabla(\nabla u)\,)
\end{equation*}
and we know that
$$\nabla g(\|\nabla u\|^{2}) = 2g'(\|\nabla u\|^{2})\nabla(\nabla u)[\nabla u]$$
and recalling the symmetry of $\nabla(\nabla u)$, we have
\begin{equation*}
\begin{split}
\nabla u\otimes \nabla g(\|\nabla u\|^{2}) &=\nabla u\otimes
2g'(\|\nabla u\|^{2})\nabla (\nabla u) [\nabla u]\\
&=2\,g'(\|\nabla u\|^{2})\left\{ \nabla u\otimes \nabla (\nabla u)
[\nabla u]\right\}\\
&=2g'(\|\nabla u\|^{2})\left(\nabla u\otimes  \nabla u\right) \nabla
(\nabla u).
\end{split}
\end{equation*}
For the last equality, we have followed~\cite[p. 9; 6(b)]{RefWorks:59}. Now taking traces in the last result we get
\begin{equation*}
\begin{split}
Tr (\,\nabla u\otimes \nabla g(\|\nabla u\|^{2})\,) &=2g'(\|\nabla
u\|^{2})Tr(\,\left(\nabla u\otimes  \nabla u\right)
\nabla (\nabla u) \,)\\
&=2g'(\|\nabla u\|^{2})\left(\nabla u\otimes \nabla
u\right)\bs{:}\nabla (\nabla u)
\end{split}
\end{equation*}
On the other hand,
$$Tr(\,g(\|\nabla u\|^{2})\,\nabla(\nabla u)\,) = g(\|\nabla u\|^{2})\, \bs{I}\bs{:}\nabla (\nabla u)$$
Therefore
\begin{equation*}
Tr (\nabla \bs{F}) =\Big[g(\|\nabla u\|^{2})\bs{I} + 2\,g'(\|\nabla u\|^{2}) \,\nabla u
\otimes \nabla u\Big]\bs{:}\nabla(\nabla u)
\end{equation*}
so we arrive to the intrinsic expression~\eqref{eq:intri_11}. \\\qed
\end{proof}

We note that the Laplacian operator $\Delta$ is intrinsic respect any basis.
For
example, it is easy to check that for 3D, we have
\begin{equation}
\begin{split}
\Delta_{3} u &= (\bs{\xi}\otimes \bs{\xi} +
\bs{\zeta}\otimes\bs{\zeta}+\bs{\eta}\otimes\bs{\eta})\bs{:}\nabla
(\nabla u)_{\widehat{B}} \\
&=(\vec{i}\otimes \vec{i} +
\vec{j}\otimes\vec{j}+\vec{k}\otimes\vec{k})\bs{:}\nabla (\nabla
u)_{B}
\end{split}
\end{equation}
where $\widehat{B}$ is the orthonormal basis previously mentioned (see equation~\eqref{eq:intri_1}).
Then,  if
we use the  orthonormal basis $\widehat{B}$ above, we obtain
\begin{equation}\label{eq:intri_3}
\partial_{t}\,u(\vec{x},t) =  \left[g(\|\nabla u\|^{2})\bs{I} + 2\,g'(\|\nabla
u\|^{2})\|\nabla u\|^{2} \,\bs{\eta} \otimes \bs{\eta}
\right]\bs{:}\nabla(\nabla
u)_{\widehat{B}}
\end{equation}
and we get,
\begin{equation*}
\begin{split}
\partial_{t}\,u(\vec{x},t) &= g(\|\nabla u\|^{2})\,u_{\bs{\xi\xi}} +
g(\|\nabla u\|^{2})\,
u_{\bs{\zeta\zeta}}\\
&+\bigg[g(\|\nabla\,u\|^{2})+2\,g'(\|\nabla
u\|^{2})\|\nabla u\|^{2} \bigg]\,u_{\bs{\eta\eta}}.
\end{split}
\end{equation*}
where $u_{\bs{\xi\xi}}, u_{\bs{\zeta\zeta}}, u_{\bs{\xi\xi}}$ are the elements of the main diagonal of $\nabla(\nabla
	u)_{\widehat{B}}$.
If $\rho$ is given by~\eqref{eq:intri_7} with $s = \|\nabla u\|$ then:
\begin{subequations}
\begin{eqnarray}
\rho'\left(\|\nabla u\|\right) &=& g\left(\|\nabla u\|^{2}\right) \, \|\nabla u\| \\
\rho''\left(\|\nabla u\|\right) &=& g\left(\|\nabla\,u\|^{2}\right)+2\,g'\left(\|\nabla u\|^{2}\right)\,\|\nabla u\|^{2} \label{eq:9}
\end{eqnarray}
\end{subequations}
then, we have another expression for $\partial_{t}\,u(\vec{x},t)$
\begin{equation}\label{eq:intri_13}
\partial_{t}\,u(\vec{x},t) = \frac{\rho'(\|\nabla u\|)}{\|\nabla u\|}\,
u_{\bs{\xi\xi}}+ \frac{\rho'(\|\nabla u\|)}{\|\nabla u\|}\,
u_{\bs{\zeta\zeta}} +\rho''(\|\nabla u\|)\,u_{\bs{\eta\eta}}
\end{equation}
with $\|\nabla u\| \neq 0$. Analogous results can be obtained for $2D$ and
$1D$ as particular cases.

An interesting insight into nonlinear diffusion filtering may also
be gained by looking at its relation with energy minimization.
If we consider the functional  given by~\eqref{eq:funcional},
with $\rho$ given by~\eqref{eq:intri_7} with $s = \|\nabla u\|$, it can be deduced that the second variation is
formally equal to:
\begin{equation}\label{eq:intri_5}
\delta^{2}\Psi(u) \equiv 2\,g'(\|\nabla u\|^{2})(\nabla u \otimes
\nabla u)  + g(\|\nabla u\|^{2})\bs{I}
\end{equation}
which it is related with the first operator of the right member of the equation~\eqref{eq:intri_11}. If $\rho$ is strictly convex, the quadratic form~\eqref{eq:intri_5}
is positive, hence the potential $\Psi$ is convex, so it has exactly
one minimum.
Nevertheless, we are interested in the case when $\rho$ is not convex. In fact, to get images that are piecewise constant and to obtain clear and sharp edges it is necessary to increase the contrast between different regions. That is why this process should be related to an enhancement process. Thus, the condition imposed
(see Aubert and  Kornprobst~\cite[p. 122]{RefWorks:11})
\begin{equation}\label{eq:8}
\rho''(\|\nabla u\|) = g(\|\nabla u\|^{2})+2\,g'(\|\nabla u\|^{2})\|\nabla u\|^{2}<0
\end{equation}
As mentioned before, this makes $\rho$ not convex. Therefore it is not guaranteed the existence of solutions of equation~\eqref{eq:intro_1}.

\subsection{Conditions for diffusivity functions}
We propose to use the result in~\eqref{eq:8} as the condition for edge enhancement diffusivity
functions~\cite{RefWorks:11}. To do this, observe that \eqref{eq:8} can be written  as
\begin{equation}\label{eq:cond_1}
\frac{2\,g'(s^{2})}{g(s^{2})}\,s^{2}
 <-1 \iff \frac{f'(s)}{f(s)}\,s < -1,
\end{equation}
where $f(s) = g(s^{2})$, which give us conditions that sho\-uld be verified by the enhancement
diffusivity functions. One way to impose this condition is to get $f(s)$ as the solution of the differential equation
\footnote{Another possibility is to consider $g(s^{2})= -\frac{s^{2}}{\tau^{2}}$ that leads to the Perona-Malik function $f(s) = e^{-s^{2}/2\tau^{2}}$}
\begin{equation*}
\frac{f'(s)}{f(s)}\,s = -p < -1
\end{equation*}
where $p>1$ and $s\neq 0$. The solution is easily obtained: $f(s) = \frac{1}{s^{p}}$ and thus,
\begin{equation}\label{eq:condition_2}
 g(s^{2}) = \frac{1}{(s^{2})^{p/2}}.
\end{equation}
Thus, based on \eqref{eq:condition_2}, we define the function called \emph{enhancement diffusivity
function}:
\begin{equation}\label{eq:condition_3}
g(r) = \frac{1}{r^{p/2}}
\end{equation}
This kind of diffusivity has been
already studied for a continuous model in the image denoising context using techniques of the porous media equations~\cite{RefWorks:56}.

\section{Pre-processing  function for segmentation task}
For the diffusivity, given by~\eqref{eq:condition_2}, it is easy to
check that
\begin{equation*}
\rho''(s) = g(s^{2}) +2\cdot g'(s^{2})\cdot s^{2} =
\frac{1}{(s^{2})^{p/2}}\,(1-p)<0 \, \forall \, p>1.
\end{equation*}

We now focus our attention in this kind of functions because with them we get
backward diffusion if $p>1$. A regularized version of this diffusivity is
\begin{equation}\label{eq:condition_4}
g(s^{2}) = \frac{1}{(s^{2} + \varepsilon^{2})^{p/2}},
\end{equation}
with $\varepsilon >0$ a fix and small constant introduced for numerical purposes. A parameter as $\varepsilon$ has been used before in the
context of denoising process, when the goal was to reduce smoothing
across the edges yielding to the Edge Enhancing Flow~\cite{RefWorks:53}; furthermore, with $p=2$, it
leads to the Balanced Forward-Backward Diffusion~\cite{RefWorks:35}.
It is clear that if $\varepsilon$ is small, the
function given in \eqref{eq:condition_4} can take large values when $\|\nabla
u\|\approx 0$. Hence, this regularization
presents unreliable results for large values of time steps $\Delta t$ when it is implemented.

The goal is to get functions with a behaviour similar to the \emph{bounded step function} $B(s)$
(see figure~\ref{fig:pulso}).
\begin{figure}[ht]
\begin{center}
\scalebox{0.7}{\includegraphics{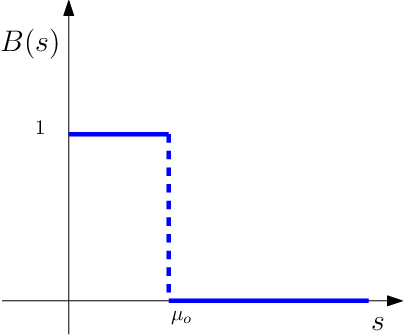}}
\caption{\label{fig:pulso} Bounded step function}
\end{center}
\end{figure}
In this case, we set linear diffusion for slopes less than
$\mu_{0}$ and no diffusion otherwise. This will allow us to establish a stopping criterion based on the \emph{setting time} for the linear diffusion. Therefore we can select the piecewise constant regions by changing the parameter $\mu_{0}$. The disadvantage of this function is that it is not continuous.

For this purpose, using the function $g(t)$  in \eqref{eq:intri_7} and its derivative, the enhancement diffusivity function gi\-ven by \eqref{eq:condition_3} and the bounded step function $B(s)$, we propose the following  diffusivity function
\begin{equation}\label{eq:4}
g_{a}(r) = \left\{\begin{array}{cl}
1 & \text{if $0\le r<\gamma$}\\[4pt]
\displaystyle \left(\frac{ \gamma }{r} \right)^{p/2}& \text{if $\gamma \le r$},
\end{array}\right.
\end{equation}
where $p>1$ is the diffusivity parameter
in~\eqref{eq:condition_3}, and $\gamma$ is a threshold parameter that allows to get linear diffusion when $r$
is less than $\gamma$, and backward diffusion when $r$ is
bigger than $\gamma$. Let us observe that $\mu_{0}$ is related to $\gamma$ as $\sqrt{\gamma} = \mu_{0}$.

In this way this diffusivity looks like the bounded step function $B(s)$
when $\mu_{0} = \sqrt{\gamma}$ and $p \to \infty$,  as  it can be seen in Fig.~\ref{subfig:F_HM}. It can also be seen that $g_{a}$ is
continuous and differentiable at any point in $[0,+\infty)$, except at $
r= \gamma$. But its derivative is bounded by $-\frac{p}{2}\, \gamma^{-1}$, so we can extend it to $r =
\gamma$ by defining $g'_{a}(\gamma) \equiv \lim_{x\to \gamma^{+}}
g'_{a}(x)$. Hence, it takes the form
\begin{equation}\label{eq:5}
\widehat{g_{a}'}(r) = \left\{\begin{array}{cl}
0 & \text{if $r<\gamma$}\\[7pt]
\displaystyle -\frac{p}{2}\,\frac{ \gamma^{p/2} }{r^{p/2+1}} &
\text{if $\gamma \le r$},
\end{array}\right.
\end{equation}
and therefore $g'_{a}(r) \stackrel{\text{a.e.}}{=} \widehat{g_{a}'}(r)$
(\emph{a.e.}: almost everywhere) (see Fig.~\ref{subfig:FD_HM}). This extension is not particulary
restrictive, from now on we will work with the function
$\widehat{g'_{a}}$ as the derivative of $g_{a}$ and,  accordingly,  we will use
the notation $g'_{a}$ instead of $\widehat{g'_{a}}$.
\begin{figure}[H]
\centering
\subfigure[$g_{a}$]{\includegraphics[width=0.47\linewidth]{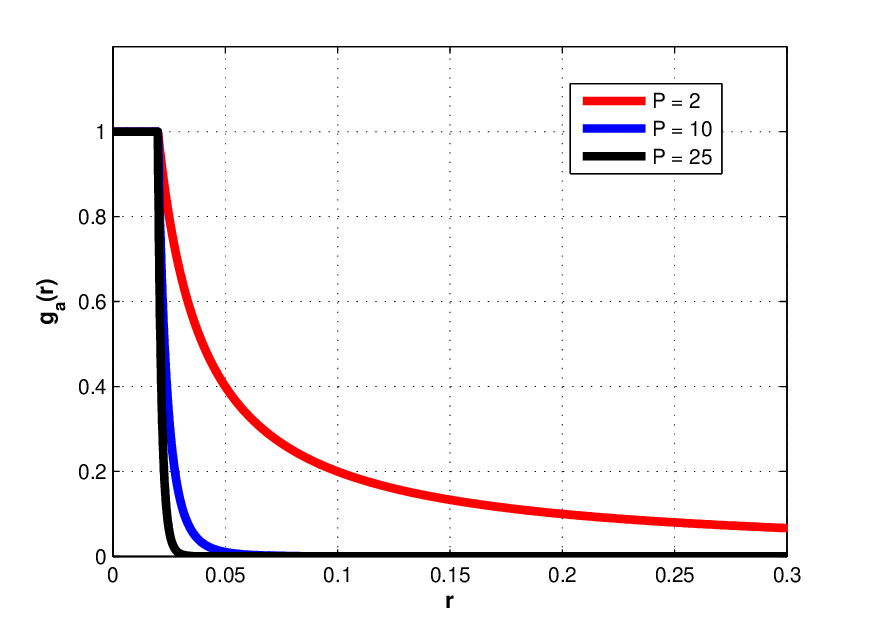}\label{subfig:F_HM}}
\hfill
\subfigure[$g'_{a}$]{\includegraphics[width=0.47\linewidth]{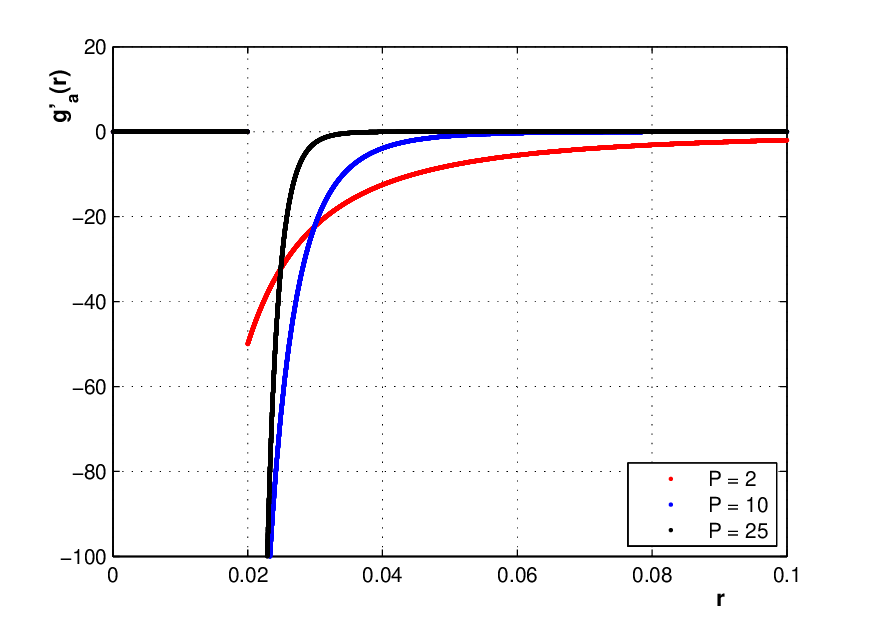}\label{subfig:FD_HM}}
\caption{\label{fig:pulso2} Graphics of the proposed diffusivity for different values of $p$ (\emph{e.g.} $\gamma = 0.02$).}
\end{figure}
We also note that $g_{a}$ has the Lipschitz property.
\begin{prop}\label{prop:1}
The function $g_{a}$ given by~\eqref{eq:4} is Lipschitz continuous in
$[0,+\infty)$ ({i.e.} $g_{a}\in \mc{C}^{0,1}([0,+\infty))$)
\end{prop}
\begin{proof}
We have to show that
\begin{equation}\label{eq:lipschitz}
|g_{a}(r_{2}) - g_{a}(r_{1})|\le L\,|r_{2} - r_{1}| \quad \forall r_{1}, r_{2} \in [0,+\infty)
\end{equation}
with $r_{1}<r_{2}$.
The proof will be divided into three steps: \emph{(a)} $r_{1},r_{2}\in [0,\gamma]$, \emph{(b)} $r_{1},r_{2}\in [\gamma, +\infty)$ and \emph{(c)} $r_{1}<r_{2}$ with $r_{1} \in [0, \gamma)$ and $r_{2}\in (\gamma ,\infty)$.
The proofs of the first two steps, are immediate using the Mean Value Theorem. We now analyze the third case.
Let us assume  that $r_{1}<r_{2}$ with $r_{1} \in [0, \gamma)$ and
$r_{2}\in (\gamma ,\infty)$. According to the results for \emph{(a)}, \emph{(b)} and using that
$|r_{2}-\gamma| \le |r_{2} - r_{1}|$ we have that
\begin{equation*}
\begin{split}
|g_{a}(r_{2}) - g_{a}(r_{1})| &= |g_{a}(r_{2})-g_{a}(\gamma) + g_{a}(\gamma) - g_{a}(r_{1})|\\
&\le |g_{a}(r_{2})-g_{a}(\gamma)| + |g_{a}(\gamma) - g_{a}(r_{1})| \\
&\le L \,|r_{2} - \gamma| \le L \,|r_{2} - r_{1}|
\end{split}
\end{equation*}
It may be concluded that
$\forall\, r_{1}, r_{2} \in [0,\infty)$,  with $r_{1}<r_{2}$, equation~\eqref{eq:lipschitz} is satisfied with $L =\frac{p}{2}\, \gamma^{-1}$. \\\qed
\end{proof}

From now on, to shorten and clarify the notation, we will write $g$ instead of $g_{a}$ for the function in \eqref{eq:4}.

\section{Numerical approximation for 1D}\label{sec:numerical1D}

\subsection{Semi-discretization procedure for 1D}
We begin  using the MOL (Method of
Lines)~\cite{RefWorks:14} for the equation~\eqref{eq:intro_1} in 1D
with the diffusivity $g$ given by~\eqref{eq:4}. Therefore, we follow a spatial
semi-discretization procedure to convert a partial differential equation into a system of
coupled ordinary differential equations. For simplicity,  we will use a uniform grid, with a grid
spacing $h$. The function $\vec{U}(t)$ will be obtained as the
solution of a system of ordinary differential equations that
comes from approximating $\partial_{x}u(x,t)$ at points halfway
between the grid points, using a centered approximation in the
original equation~\eqref{eq:intro_1}. In matrix form, the system of autonomous
ordinary differential equations with Neumann boundary
conditions is
\begin{equation}\label{eq:7}
\frac{d\vec{U}}{dt} (t) = \bs{f}(\vec{U}(t)) =
\bs{A}(\vec{U}(t))\,\vec{{U}}(t)
\end{equation}
where $\bs{A} = [a_{ij}]$ has the following coefficients, for $i=2,\ldots,m-1$
\begin{subequations}\label{eq:7a}
\begin{align}
\displaystyle a_{i,i+1} =  \frac{g_{i}}{h^{2}}  = a_{i+1, i}; & \quad
\displaystyle a_{i,i}= -\frac{g_{i-1} +
g_{i}}{h^{2}}
\end{align}
\begin{equation}
\displaystyle a_{i, i-1} =
\frac{g_{i-1}}{h^{2}}=a_{i-1, i}
\end{equation}
\end{subequations}
being $ g_{i} = g\left(\frac{1}{h^{2}}|U_{i+1}-U_{i}|^{2}\right)$
with $i=1,\ldots, m-1$ and for the Neumann boundary conditions
$$a_{1, 1} =  -\frac{1}{h^{2}}\, g_{1};\quad
a_{m, m} =  -\frac{1}{h^{2}}\, g_{m-1}$$ and $a_{i,j} = 0$ otherwise.

The next step is to check whether the system of differential equations~\eqref{eq:7}, with the diffusivity function proposed in \eqref{eq:4}, satisfies the well-posedness and semi-dis\-cre\-te scale-space requirements~\cite{RefWorks:130}. For this reason, the matrix $\bs{A}(\vec{U})$ must meet certain requirements as were established in  \cite{RefWorks:130} by Weickert (see $(S1)\;  \text{to} \; (S5)$ in page 76).
It can be  easily proven that  $\bs{A}(\vec{U})$ is a real tridiagonal and
symmetric matrix (\cite[($S2$)]{RefWorks:130}). As a consequence of the positiveness of
diffusitivity ($g>0$) given by~\eqref{eq:4},  $a_{i,\, i+1}$ and
$a_{i+1,\, i}$ are different from zero and non-negative (\cite[($S4$)]{RefWorks:130}). Hence the matrix is
irreducible~\cite{RefWorks:81} (\cite[($S5$)]{RefWorks:130}).  It is clear that the
rows sum are zero and the off-diagonal terms are non-negative (\cite[($S3$)]{RefWorks:130}).  Now we
show that $\bs{A}(\vec{U})$ has the Lipschitz condition
for~\eqref{eq:4} for every bounded subset (\cite[($S1$)]{RefWorks:130}).
\begin{prop}\label{prop:lips}
The matrix $\bs{A}(\vec{U}) =[a_{ij}(\vec{U})]$ defined in \eqref{eq:7} with
the  diffusivity function $g$~\eqref{eq:4}, satisfies the Lipschitz
condition for any bounded subset of $\mathbb{R}^{m}$.
\end{prop}
\begin{proof}
For any $K>0$ we consider the set given by
\begin{equation}\label{eq:set1}
\Upsilon = \{\vec{U} \in \mathbb{R}^{m}\; | \; \|\vec{U}\|_{1} \le K\}
\end{equation}
By the \textsf{Proposition~\ref{prop:1}} above,
\begin{multline*}
\left|g\left(\frac{1}{h^{2}}(U_{i+1}-U_{i})^{2}\right) -
g\left(\frac{1}{h^{2}}(V_{i+1}-V_{i})^{2}\right)\right|\\ \le
\left|\frac{L}{h^{2}} \left[ (U_{i+1}-U_{i})^{2} - (V_{i+1} -
V_{i})^{2} \right] \right|
\end{multline*}
but
\begin{multline*}
\left[ (U_{i+1}-U_{i})^{2} - (V_{i+1} - V_{i})^{2}
\right]\\
=\left[(U_{i+1}-U_{i})+(V_{i+1}-V_{i})\right]\\
\cdot\left[(U_{i+1}-V_{i+1})+(V_{i}-U_{i})\right]
\end{multline*}
and
$$0\le |U_{i+1} - U_{i}| \le K \quad i=1,\ldots,m-1 \quad \forall \,\vec{U}\in \Upsilon$$
hence
\begin{multline}\label{eq:olga}
\left|g\left(\frac{1}{h^{2}}(U_{i+1}-U_{i})^{2}\right) -
g\left(\frac{1}{h^{2}}(V_{i+1}-V_{i})^{2}\right)\right|\\ \le
\frac{2K\, L}{h^{2}} \bigg[\; |U_{i+1}-V_{i+1}| + |U_{i} - V_{i}|
\;\bigg].
\end{multline}
Now, let us denote by $\bs{A}(\vec{U}) = [a_{i,j}] \quad \text{and} \quad \bs{A}(\vec{V}) =
[\tilde{a}_{i,j}]$, then
\begin{equation*}
\| \bs{A}(\vec{U}) - \bs{A}(\vec{V}) \|_{1} = \max_{j}
\sum_{i=1}^{m} |a_{i,j}- \tilde{a}_{i,j}|.
\end{equation*}
The two matrices are symmetric and tridiagonal, therefore
\begin{multline*}
\|\bs{A}(\vec{U}) - \bs{A}(\vec{V}) \|_{1}
=\max_{j}\left[|a_{j,j-1}-\tilde{a}_{j,j-1}| +
|a_{j,j}-\tilde{a}_{j,j}| \right.\\ \left. + |a_{j, j+1}-\tilde{a}_{j,j+1}| \right].
\end{multline*}
According to~\eqref{eq:7a} and~\eqref{eq:olga},  we get:
\begin{multline*}
|a_{j,j}-\tilde{a}_{j, j}| \le  \frac{1}{h^{2}}\, |g_{j-1} -
\tilde{g}_{j-1}|+ \frac{1}{h^{2}}\, |g_{j} - \tilde{g}_{j}|\\
\le  \frac{2\,K\, L}{h^{4}} \bigg[\;|U_{j+1} - V_{j+1}|+
2\,|U_{j}-V_{j}| +
|U_{j-1} - V_{j-1}| \;\bigg]\\
\le\frac{4\,K\, L}{h^{4}}\|\vec{U} - \vec{V}\|_{1}.
\end{multline*}
Similar arguments apply to the cases
\begin{multline*}
|a_{j,j-1}-\tilde{a}_{j,j-1}| \le \frac{2\, L}{h^{4}} \|\vec{U} -
\vec{V}\|_{1} \quad \text{and} \\   |a_{j,j+1}-\tilde{a}_{j,j+1}|
\le \frac{2\,K\, L}{h^{4}}\|\vec{U} - \vec{V}\|_{1}.
\end{multline*}
Finally, we get
\begin{equation*}
\|\bs{A}(\vec{U}) - \bs{A}(\vec{V}) \|_{1}\le \, M\|\vec{U} - \vec{V}\|_{1}
\end{equation*}
being
$$M =\frac{8\,K\,L}{h^{4}}\quad \text{with}\quad L =\frac{p}{2}\,
\gamma^{-1}
$$
Since there has been no restriction on the set $\Upsilon$, it must be true for any bounded set of $\mathbb{R}^{m}$.\\ \qed
\end{proof}

Therefore, following the results by Weic\-kert and Be\-nha\-mouda in \cite{RefWorks:31},
the semidiscrete~\eqref{eq:7} process for
diffusivities $g$ \eqref{eq:4}, satisfies the well-posedness and
semi-discrete sca\-le-spa\-ce requirements~\cite{RefWorks:130}.

\subsection{Full discretization  for 1D}
We now proceed to discretize the temporal variable in the ordinary
differential system~\eqref{eq:7}. We have two possibilities, one is to set an explicit scheme. The drawback of this method
is the restriction for time step size~\cite{RefWorks:42}.
For this reason, we prefer to use implicit schemes. In particular we use
the Backward Euler method. In this case we have:
\begin{equation}\label{eq:impli_2}
\frac{\vec{U}^{n+1}-\vec{U}^{n}}{k}
=\bs{A}(\vec{U}^{n+1})\vec{U}^{n+1} =\vec{f}(\vec{U}^{n+1})
\end{equation}
where $\vec{f}$ is given by~\eqref{eq:7} and we approximate $\vec{U}(t_{n})= \vec{U}(nk)$ by $\vec{U}^{n}$
with $k\equiv \Delta t$, the time step size. We can easily realize that~\eqref{eq:impli_2} is a nonlinear algebraic equation. One straightforward
possibility to solve it, is to use the Newton-Raphson method. In this case, we consider the function
\begin{equation}\label{eq:impli_10}
\vec{q}(\vec{U}) = \vec{U}- k\,\vec{f}(\vec{U}) - \vec{U}^{n}=
\left[\bs{I}- k\,\bs{A}(\vec{U})\right]\vec{U} - \vec{U}^{n}
\end{equation}
and the Newton-Raphson's method is the iterative me\-thod
\begin{equation}\label{eq:newton}
\vec{U}^{\nu+1} = \vec{U}^{\nu}-\Big(\bs{I}-k\,
\frac{\partial\,\vec{f}}{\partial\vec{U}}(\vec{U}^{\nu})\Big)^{-1}\vec{q}(\vec{U}^{\nu})
\end{equation}
with $\nu = 0,1,\ldots$ and where $\displaystyle \Big(\bs{I}-k\,
\frac{\partial\,\vec{f}}{\partial\vec{U}}(\vec{U}^{\nu})\Big)$ is
the iteration matrix. To construct explicitly this iteration
matrix we give an expression for the Jacobian matrix of $\vec{f}$ given by~\eqref{eq:7}.
\begin{prop}\label{prop:f}
The Jacobian matrix of $\vec{f}$ defined by \eqref{eq:7} can be
expressed by a sum of two real tridiagonal and symmetric matrices
\begin{equation}\label{eq:discre_5}
\frac{\partial\,\vec{f}}{\partial\vec{U}}(\vec{U}) = \bs{A}(\vec{U})
+ \bs{C}(\vec{U}).
\end{equation}
\end{prop}
\begin{proof}
By definition of function the $\vec{f}(\vec{U}) = \bs{A}(\vec{U})\,\vec{U}$ and
using derivation rules, yields
\begin{equation*}
\bs{\Fb{J}} = \frac{\partial\,\vec{f}}{\partial\vec{U}}(\bs{U}) =
\bs{A}(\bs{U})+ \frac{\partial \bs{A}(\bs{U})}{\partial \bs{U}}
\bs{U}.
\end{equation*}
For simplicity of notation, we use\footnote{Note Einsten's criterion}
\begin{equation*}
\bs{C}(\vec{U})\equiv \frac{\partial \bs{A}(\bs{U})}{\partial
\bs{U}}\, \bs{U}, \quad \text{so}\quad
[c_{i,l}]=\left[\frac{\partial a_{i,j}}{\partial U_{l}}\,
U_{j}\right],
\end{equation*}
and $f_{i,j}$ for the Jacobian coefficients. It is easy to check that
$$\frac{\partial a_{i, i-1}}{\partial U_{l}} = 0 \quad \text{if} \; l\neq i-1 \; \text{or} \; l \neq i,$$
since $a_{i, i-1} = (1/h^{2})\, g_{i-1}$ only depends on $U_{i}$
and $U_{i-1}$. Similarly
$$\frac{\partial a_{i, i+1}}{\partial U_{l}} = 0 \quad \text{if} \; l\neq i+1 \; \text{or} \; l \neq i.$$
This gives the follow terms, not necessarily null:
\begin{eqnarray*}
c_{i,i+1} &=& \frac{2}{h^{4}}\,g'_{i}\cdot(U_{i+1}-U_{i})^{2}  \quad \text{with} \, i = 1,\ldots, m-1,\\
c_{i,i} &=&  -\frac{2}{h^{4}}\left[g'_{i}\cdot(U_{i+1}-U_{i})^{2}+g'_{i-1}\cdot(U_{i}-U_{i-1})^{2}\right]\\
  && \text{with} \, i = 2,\ldots, m-1,\\
c_{i,i-1} &=& \frac{2}{h^{4}}\,g'_{i-1}\cdot(U_{i}-U_{i-1})^{2}
\quad \text{with} \, i = 2,\ldots, m
\end{eqnarray*}
and
\begin{gather*}
c_{1,1} = -\frac{2}{h^{4}}\,g'_{1}\cdot(U_{2}-U_{1})^{2};\\
c_{m,m} = -\frac{2}{h^{4}}\,g'_{m-1}\cdot(U_{m}-U_{m-1})^{2},
\end{gather*}
where
\begin{equation}\label{eq:discre_10}
g'_{i} = g'\left(\frac{1}{h^{2}}(U_{i+1}-U_{i})^{2}\right)\qquad i =
1,\ldots, m-1.
\end{equation}
We also note that $c_{i,i+1}=c_{i+1,i}$ since
\begin{gather*}
\frac{\partial\,a_{i, i+1}}{\partial\,
U_{i+1}}U_{i+1} = \frac{\partial\,a_{i+1,
i+1}}{\partial\,U_{i}}U_{i+1}, \quad \text{and}\\
\frac{\partial\,a_{i, i}}{\partial\, U_{i+1}}U_{i} =
\frac{\partial\,a_{i+1, i}}{\partial\,U_{i}}U_{i}.
\end{gather*}
Similar arguments apply to the case  $c_{i-1,i} = c_{i,i-1}$. Consequently
we can establish the relation~\eqref{eq:discre_5}.\\ \qed
\end{proof}

It is worth pointing out several consequences that can be extracted from the last
proposition. Firstly, the expression~\eqref{eq:discre_5} is directly
connected with the term
\begin{equation}\label{eq:newton2}
g(\|\nabla u\|^{2})+2\,g'(\|\nabla u\|^{2})\|\nabla u\|^{2}
\end{equation}
appearing in equations~\eqref{eq:intri_11} or~\eqref{eq:9}. That is,
according to the matrix elements $a_{i,j}$ and $c_{i,j}$ of the expression \eqref{eq:discre_5} we can relate these terms with \eqref{eq:newton2} in the following: the coefficients of matrix $\bs(A)$ only use the function $g$, as we can see in \eqref{eq:7a}, while the coefficients of matrix $\bs(C)$ only use the function $g'$, as we can see in the proof of Proposition \ref{prop:f}.
Secondly, the
matrix $\bs{C}(\vec{U})$ is well defined because the derivative of the proposed diffusivity~\eqref{eq:4} exists at any point in $[0,\infty)$ as we have established in~\eqref{eq:5}.

On the other hand, we also note that if $g$ is a monotone
decreasing function,  its derivative $g'$ is non-po\-si\-ti\-ve. This implies
that the $c_{i,j}$ coefficients are al\-so non-po\-si\-ti\-ve and the coefficients of the Jacobian
could change their signs and thus the signs of the eigenvalues of the
Jacobian  matrix in the iterative process. This can make the iteration matrix singular and, hence, some important
instabilities in the numerical solution might appear.

To avoid this drawback, we apply the \emph{tangential stiffness
method} that will connect with an edge-pre\-ser\-ving pro\-cess. This scheme has already been used in nonlinear Finite
Element analysis context for nonlinear  mechanics of solids and
plasticity problems~\cite{RefWorks:92,RefWorks:93}. We propose using
this method to solve the nonlinear equation~\eqref{eq:impli_2}.
We give the following approximation to the Newton-Raphson iteration matrix~\eqref{eq:newton}
\begin{equation}\label{eq:aprox_1}
\bs{I}- k\frac{\partial\vec{f}}{\partial\vec{U}}\left(\vec{U}\right)
= \bs{I} - k\bs{A}(\vec{U})-k\bs{C}(\vec{U}) \approx \bs{I} -
k\bs{A}(\vec{U})
\end{equation}
Substituting this approximation in expression~\eqref{eq:newton} and using~\eqref{eq:impli_10}, we arrive at the following iterative process:
\begin{equation}\label{eq:iterativo}
\vec{U}^{\nu+1} = \left[\bs{I}-k\bs{A}(\vec{U}^{\nu})\right]^{-1}\vec{U}^{n}
\end{equation}
with $\nu = 0,1,2,\ldots \quad \forall\, k\in \mathbb{R}^{+}$
where we choose $\vec{U}^{n}$  for $\nu = 0$ as  first iteration
term. The advantage of using this iteration matrix lies in the fact
that  $\bs{W}_{k}(\vec{U}) \equiv
\left[\bs{I}-k\bs{A}(\vec{U})\right]$  is positive definite by the properties of $g > 0$ for all $\vec{U}$ and $k\in
\mathbb{R}^{+}$ (see Ortega~\cite[p. 107]{RefWorks:91}).

Since $\bs{W}_{k}(\vec{U})$ is positive definite,
$\bs{W}_{k}(\vec{U})^{-1}$ exists. We also know
that $\bs{W}_{k}(\vec{U})$ satisfies $w_{i,j} \le 0$, $i\neq j$ and
$w_{i,i}>0$, $i=1,\ldots, m$ and is symmetric, then
by~\cite[p. 110]{RefWorks:91}, it is a \emph{Stieltjes matrix}
so, $\bs{W}_{k}(\vec{U})^{-1}\ge O$ (see~\cite{RefWorks:91}).

We now consider $[\vec{V}]^{t} =
[1,1,\ldots,1]$ and since $\bs{A}(\vec{U})$ has unit row sum, we
have $[\vec{V}]^{t}\bs{W}_{k}(\vec{U})=[\vec{V}]^{t}$ and it yields
$[\vec{V}]^{t}\bs{W}_{k}(\vec{U})^{-1}=[\vec{V}]^{t}$. Then we also can conclude that
\begin{equation*}
\|\bs{W}_{k}(\vec{U})^{-1}\|_{1} \equiv
\bigg\|\left[\bs{I}-k\bs{A}(\vec{U})\right]^{-1}\bigg\|_{1} = 1
\quad \forall \, k \in \mathbb{R}^{+}
\end{equation*}
Because of $\left[\bs{I}-k\,\bs{A}\left(\vec{U}\right)\right]^{-1}$
is symmetric, we also have
$\left\|\left[\bs{I}-k\,\bs{A}\left(\vec{U}\right)\right]^{-1}\right\|_{\infty}
= 1$
\begin{prop}
For the iterative process~\eqref{eq:iterativo}, we have
\begin{equation*}
\|\vec{U}^{\nu+1}\|_{1} \le \|\vec{U}^{n}\|_{1} \qquad \nu =
0,1,2,\ldots  \qquad \forall\, k \in \mathbb{R}^{+}
\end{equation*}
where we choose $\vec{U}^{n}$ as the first iteration.
\end{prop}
\begin{proof}
The proof is straightforward. By
$$\left\|\left[\bs{I}-k\,\bs{A}\left(\vec{U}\right)\right]^{-1}\right\|_{1} = 1\quad \text{it follows that}$$
\begin{equation*}
\|\vec{U}^{\nu+1}\|_{1} \le
\left\|\left[\bs{I}-k\,\bs{A}\left(\vec{U}^{\nu}\right)\right]^{-1}\right\|_{1}\|\vec{U}^{n}\|_{1}
= \|\vec{U}^{n}\|_{1}
\end{equation*}
with $\nu = 0,1,2,\ldots \quad \forall\, k
\in \mathbb{R}^{+}$.
This establishes that the iterative process is  stable $\forall\, k \in \mathbb{R}^{+}$. \\ \qed
\end{proof}

The iterative scheme~\eqref{eq:iterativo} is also known as the
\emph{direct iteration method}, the \emph{successive approximation
method} or the \emph{Picard method}~\cite{RefWorks:92,RefWorks:93}.

We emphasize  several points about the iterative
process in \eqref{eq:iterativo}.
\begin{enumerate}
\item \label{condiciones2} If we only take one iteration, we get the
\emph{semi-im\-pli\-cit me\-thod}~\cite{RefWorks:42}. In this case, from the above results, it can directly prove that
the iterative process in \eqref{eq:iterativo}, with $g$ given by~\eqref{eq:4},
verifies the following criterion~\cite{RefWorks:42}: $(D1)$
\emph{Continuity in its argument}, $(D2)$ \emph{Symmetry}, $(D3)$
\emph{Unit row sum}, $(D4)$ \emph{Positive diagonal}, $(D5)$
\emph{Irreducibility}. Under these conditions, the iterative
process~\eqref{eq:iterativo} with only one time step, creates a
discrete scale-space~\cite{RefWorks:130}. So,  it has proved that the function of diffusivity proposed in~\eqref{eq:4} satisfies the well-posedness and semi-discrete and full discrete scale-space requirements.

\item According to the approximation given by~\eqref{eq:aprox_1} we do not construct the inverse diffusion numerically,
we only do forward diffusion. We do not use the matrix $\bs{C}$ because it involves
the derivative of the diffusivity $g$ and it is non-po\-si\-ti\-ve as we have mentioned before. However, $g$ is a
decreasing function, which implies diffusion with different speed. That is, for those pixels that satisfy the condition  $\|\nabla u\|^{2} <\gamma$, the proposed diffusion function $g$~\eqref{eq:condition_2} reaches its maximum  value,
$g=1$, we get the linear diffusion. However, when we have pixels that satisfy the condition $\gamma\le \|\nabla u\|^{2} $, then the proposed diffusion function $g$~\eqref{eq:condition_2} takes a value between $0$ and $1$ and it produces  slower
diffusion for those pixels. Thus, two contiguous pixels can be separated because one of them can produce a value of $g=1$ and
the other one $0<g<1$. This is a consequence of using an edge-preserving process.


\item From the iterative process~\eqref{eq:iterativo}, we can establish the following iterative scheme in the compact and convex set $\Upsilon$ \eqref{eq:set1}
\begin{equation}\label{eq:aprox_3}
\vec{U}^{\nu+1} = \bs{G}(\vec{U}^{\nu}),\; \text{with} \;
\bs{G}(\vec{U}^{\nu}) = \left[\bs{I}-k\bs{A}(\vec{U}^{\nu})\right]^{-1}\vec{U}^{n},
\end{equation}
for $\nu=0,1,2\ldots$ that easily yields a \emph{fixed point iteration}. Obviously a
fixed point for~\eqref{eq:aprox_3}, that is $\vec{U}^{*} =
\bs{G}(\vec{U}^{*})$ is equivalent to the equation
$\vec{q}(\vec{U}^{*}) = \vec{0}$. It is not our purpose to study
this iterative process when $\nu \ge 1$, because it is not clear
that this discrete scheme could be connected with a filtering
process and create a discrete
space-scale~\cite{RefWorks:42,RefWorks:130}. However, this result is interesting because it allows us to say that the iterative process~\eqref{eq:aprox_3} has at least one solution, without putting any restrictions on the parameter $k$.
\end{enumerate}

\section{Numerical approximation for N dimension}\label{sec:numerical2D}
We now extend the iterative process~\eqref{eq:iterativo} for two and
three space dimensions and in the remainder of this section we
assume $g$ to be the diffusivity defined by~\eqref{eq:4}. In this
case, we follow a similar process as in the case of 1D.

\subsection{Semi-discretization}
We again apply the \emph{method of lines} to the
equation~\eqref{eq:intro_1} for 2-D or 3-D. On the rectangular or
brick domain  we establish a spatial grid. Therefore, the
semidiscrete system has dimension $m = m_{1}\cdots m_{s}$, where
$s=2$ or $3$. Again, solving the semi-dicrete problem means finding
a function $\vec{U}(t)$ which is approximation to the solution $u$
at discrete grid points $(x_{1},\ldots, x_{s}, t)$. The function
$\vec{U}(t)$ will be obtained as the solution of a system of
ordinary differential equations that it comes  from approximating
$\partial_{x_{r}}u_{\alpha}$ (where $r = 1,\ldots,s$ and
$\alpha=(i,j)$ or $\alpha=(i,j,k)$) at points halfway between the
grid points, using a centered approximation in the original
equation~\eqref{eq:intro_1}.

We only use one index for pixel numbering, so we can represent the
whole image of size $m_{1}\times \cdots \times m_{s}$ by a vector of
size $m =m_{1}\cdots m_{s}$. In this case, we have again an ordinary
differential system of type~\eqref{eq:7} where the coefficient
matrix $\bs{A}(\vec{U})$ is now a block tridiagonal matrix of order
$m$. It is straightforward to check that $\bs{A}(\vec{U})$ satisfies
the conditions of Lipschitz-continuity, symmetry, vanishing row
sums, nonnegative off-diagonals and finally,  because of its associated
directed graph is stron\-gly connected (see Ortega~\cite{RefWorks:91}),
it verifies the condition of irreducibility. Therefore,
this differential system of type~\eqref{eq:7} also satisfies the
prerequisites to be a filtering process (see
Weickert~\cite{RefWorks:130}).

It is worth pointing out that $\bs{A}(\vec{U})$ can be decomposed in
the form
\begin{equation}\label{eq:matcoef2}
\bs{A}(\vec{U}) = \sum_{r=1}^{s}\bs{A}_{r}(\vec{U})
\end{equation}
where each $\bs{A}_{r}(\vec{U})$ represents one-dimensional
semi-dis\-cre\-ti\-za\-tion of the diffusion process along the $x_{r}$ axes.
It is immediate that $\bs{A}_{r}(\vec{U})$ with $r=1,\ldots,s$
verify again the properties of Lipschitz-continuity, symmetry,
vanishing row sums and nonnegative off-diagonals.

\subsection{Time discretization}
We can  extend to 2D or 3D the full discretization process developed for
1D. Therefore, from the semi-discrete system~\eqref{eq:7} with the
block-tridiagonal matrix $\bs{A}(\vec{U})$ as coefficient matrix, we
arrive to the nonlinear equation~\eqref{eq:impli_2}. We also use the
\emph{tangential stiffness method} to solve it. It yields to a
similar iterative process~\eqref{eq:iterativo} that we have
established for 1D:
\begin{equation}\label{eq:iterativo_nd}
\vec{U}^{\nu+1} = \Big[\bs{I}-k\,
\bs{A}(\vec{U}^{\nu})\Big]^{-1}\vec{U}^{n}\quad \text{with } \;\nu =
0,1,\ldots
\end{equation}
It is straightforward to show that the \emph{iteration matrix}
$\bs{W}_{k}(\vec{U}) = \bs{I}-k\bs{A}(\vec{U})$ has also the
properties that we have proved for 1D and it is immediate to check
that $\|\left[\bs{I}-k\bs{A}(\vec{U})\right]^{-1}\|_{1} =1$ too.
Finally, it follows easily that
$\left[\bs{I}-k\bs{A}(\vec{U}^{\nu})\right]^{-1}$ satisfies the
requirements of continuity in its arguments, symmetry, unit row sum,
positive diagonal and irreducibility that it can be done in a
similar way as the 1D case. Thus, for only one iteration, that is:
$\vec{U}^{n}$ for $\nu = 0$ (semi-implicit method) we get the
scale-space requirements~\cite{RefWorks:130}.

One possibility to solve the iterative
process~\eqref{eq:iterativo_nd} is the iterative scheme
\begin{equation}\label{eq:iterativo2}
\vec{U}^{\nu+1} =\frac{1}{s}\sum_{r=1}^{s}\left[\bs{I} -
s\,k\bs{A}_{r}\left(\vec{U}^{\nu}\right)\right]^{-1}[\vec{U}^{n}]
\equiv \bs{Q}\left(\vec{U}^{\nu}\right)[\vec{U}^{n}]
\end{equation}
with $\nu = 0,1,\ldots$ where we get the \emph{AOS
scheme}~\cite{RefWorks:42} for $\nu = 0$.

For one iteration, we take $\vec{U}^{n}$ for $\nu= 0$,  then $\bs{Q}(\vec{U}^{n})$ satisfies again
the properties of continuity in its arguments, symmetry, unit row
sum, positive diagonal and irreducibility  and consequently the
iterative process $\vec{U}^{n+1} = \bs{Q}(\vec{U}^{n})[\vec{U}^{n}]$
creates a discrete scale-space, for the propose diffusivity $g$~\eqref{eq:4}.

It is worth pointing out that the same conclusion that we got for 1D can
be drawn for 2D and 3D. That is, as a consequence  of the iterative
process~\eqref{eq:iterativo_nd}, we again do not use the derivative
of $g$ for discrete scheme because we do not use the matrix
$\bs{C}$, so we do not backward diffusion from the numerical point
of view, but forward diffusion with a decreasing filter $g$.

On the other hand, with the discrete scheme~\eqref{eq:iterativo2},
we only can ensure again, a discrete scale-space process for $\nu =
0$ and for each time step $k>0$, as we did for 1D.

\clearpage
\section{Experiments}
In this section, we report some experimental results obtained when we apply our filter to natural gray-scale images chosen from
a public database~\cite{RefWorks:196}. 
The aim of the experiments is to show the performance of our filter to obtain piecewise constant images as a previous step to the segmentation process. In the proposed method, four key parameters are investigated, the first two are specific to the numerical method: the stopping time $T$ and the time step $k$. The other two are needed to define the diffusivity: the edge strength threshold $\gamma$ and the parameter $p$.

We take into account a two dimensional discrete image of size $m_1 \times m_2$ as a vector ${\boldsymbol U} \in \mathbb{R}^m$, $m=m_1 \times m_2$, whose components, $U_i$, $i=1, \ldots, m$, represent the grey values at each pixel. Let ${\boldsymbol U^{(0)}}$ be the original image, then the discrete scale-space given by \eqref{eq:iterativo2}, is used to  calculate a sequence $({\boldsymbol U^{(j)}})_{j \in \mathbb{N}}$ of filtered versions of ${\boldsymbol U^{(0)}}$ by gradually removing noise and details until the image converges to the average grey level $\mu = \frac{1}{m}\sum_{i=1}^m U^{(0)}_i$, as $t \to + \infty$.

\subsection{Stopping time $T$ and time step $k$}
For nonlinear diffusion, the stopping time has a strong effect on the diffusion result and, as far as we know, there is not a straightforward stopping criterion. Several studies have addressed the stopping time problem in the field of PDEs image restoration. In \cite{RefWorks:186,RefWorks:188,RefWorks:37,RefWorks:189,RefWorks:190,RefWorks:191,RefWorks:133,RefWorks:192,solo2001automatic,RefWorks:193,RefWorks:194}, criterion on the stopping time selection are closely linked to the noise-filtering problem. In these works, the aim is to stop the process before the structure of the image has been modified too much while textured structures must be maintained.

In our case, the aim is to achieve piecewise constant images as a previous process of segmentation techniques, therefore we must determine a stop\-p\-ing ti\-me not only to eliminate image noise but noise and texture as well.
The proposed stopping time takes into account that the diffusivity we use is {\em linear} when the modulus of the gradient is less than $\sqrt{\gamma}$, and tends to zero when the modulus of the gradient is greater than $\sqrt{\gamma}$, thus instead of solving the nonlinear problem, we solve the linear diffusion problem and consider the setting time, $t_s$,
\begin{equation}\label{eq:tiempo}
\frac{\|{\boldsymbol U^{(n)}}- {\boldsymbol \mu} \|}{\|{\boldsymbol \mu}\|} \le 0.02 \qquad t_s=n \cdot k
\end{equation}
where ${\boldsymbol \mu} = (\mu, \ldots, \mu)^{T} \in \Mb{R}^{m}$, $n$ is the number of iterations performed and ${\boldsymbol U^{(n)}}$ is the filtered image using linear diffusion. We note that the proposed stopping time, $T=t_s=n \cdot k$,  operates without prior knowledge of the geometric or statistical structure of the image because it depends only on the size and pixel intensities of the initial image ${\boldsymbol U^{(0)}}$. Since the amount of work involved in the method is proportional to the number of individual steps, we attempt to choose the time step $k$ as large as possible, but such that no numerical artifacts appear when we solve the nonlinear problem using the decomposition method AOS. Experimentally we found that the time step size can take high values without noticeable difference in the results.

Given the proposed nonlinear diffusion filter, an original image is composed of several regions and choose the parameters $p$ and $\gamma$ as described in the following subsection, the filtered image obtained on the scale $t_s$ is an image in which regions are homogenized, while the edges are not blurred by the strength of the diffusion function. For example, the regions of the images in Fig.~\ref{fig:Ideales} are homogenized with the proposed filter while the edges are not blurred.
\begin{figure}[H]
\centering
\subfigure[]{\includegraphics[width=0.48\textwidth]{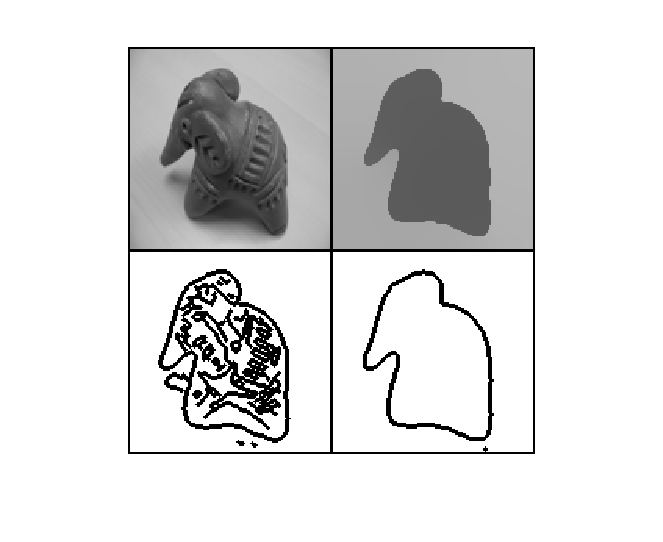}\label{subfig:Elephantk25}}
~
\subfigure[]{\includegraphics[width=0.48\textwidth]{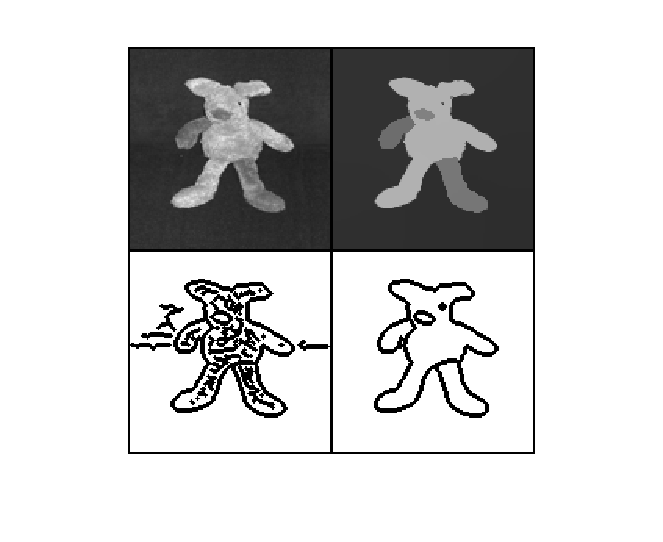}\label{subfig:Teddyk25}}
\caption{\label{fig:Ideales} Top row: The left one, show the original image: ${\boldsymbol U^{(0)}}$. The right one, show the last iteration performed: ${\boldsymbol U^{(n)}}$. The image size was $128 \times 128$ and the number of iterations, $n$, was obtained from \eqref{eq:tiempo}. Bottom row: The Canny detector, $(\sigma=1)$, in order to compare the edges detected in the initial image and the filtered image.}
\end{figure}

\subsection{Diffusivity parameters: $\gamma$ and $p$}

The parameter $\gamma$ has been chosen taking into account the statistical interpretation of nonlinear diffusion from the point of view of robust statistics developed in~\cite{RefWorks:157}.
The authors show that nonlinear diffusion can be seen as a robust estimation procedure that estimates a piecewise constant image from a noisy image where the bo\-un\-da\-ries between the piecewise regions are considered to be outliers. We apply this statistical interpretation to select $\gamma = (1.4826 \cdot MAD(\| \nabla({\boldsymbol U^{(0)}})\|))^2$, where $MAD$ is the median absolute deviation.

The diffusivity is defined so that the filter is linear when $\| \nabla({\boldsymbol U^{(j)}})\| \le \sqrt{\gamma}$  and blurs the regions and when  $\| \nabla({\boldsymbol U^{(j)}})\| > \sqrt{\gamma}$, the diffusion function $g(\| \nabla({\boldsymbol U^{(j)}})\|)$ tends to zero as shown in Fig.~\ref{subfig:F_HM}. We also note that high $p$ values in~\eqref{eq:4} prevent edge blurring because $p$ tends to the \emph{bounded step function} (Fig.~\ref{fig:pulso}). Then, in the experiments, a piecewise constant image should be obtained for low values of $p$ if the contrast among the objects and the background is  high. The value of $p$ should be higher if the image has more details or the contrast between the objects and the background is low. We apply the diffusion model to three images in order to show the effects of different values of $p$.

In Fig.~\ref{fig:Variosp} we see that the background of the flower image has more details than the bear image, then for the bear we obtained piecewise constant images when $p=2.5$ or $3$ and for the flower when $p=5$ or $5.5$. In the case of the bush image, Fig.~\ref{fig:BushVariosp} shows that the contrast among the different objects is low and to avoid blurring the value of $p$ must be greater than in previous cases, $p=19$ or $20$.

\begin{figure}[H]
\centering
\includegraphics[width=\linewidth]{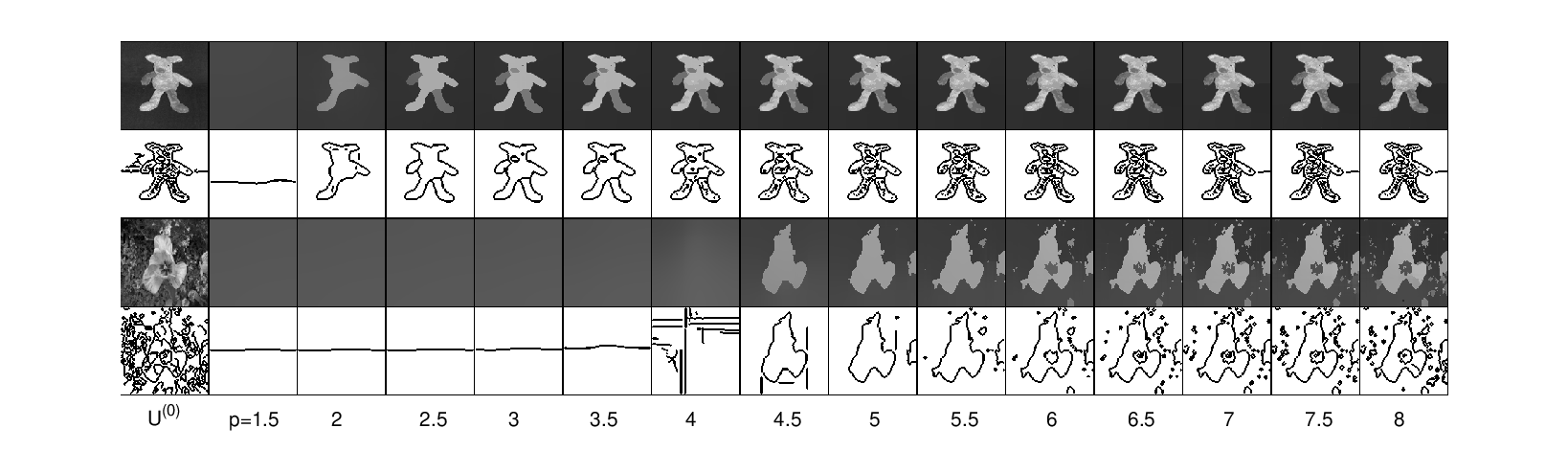}
\caption{\label{fig:Variosp} For $p=1.5,\, 2, \, 2.5, \,  \ldots ,8$ we show the original images ${\boldsymbol U^{(0)}}$ and ${\boldsymbol U^{(n)}}$ which correspond to the last iteration performed. The image size was $128 \times 128$ and the number of iterations, $n$, was obtained from \eqref{eq:tiempo}. The time step was $k=200$ and the iterations performed were $n=39$ for the bear image and $n=42$ for the flower image. The edges were detected using Canny detector $(\sigma=1)$.}
\end{figure}

\begin{figure}[H]
\centering
\includegraphics[width=\linewidth, keepaspectratio]{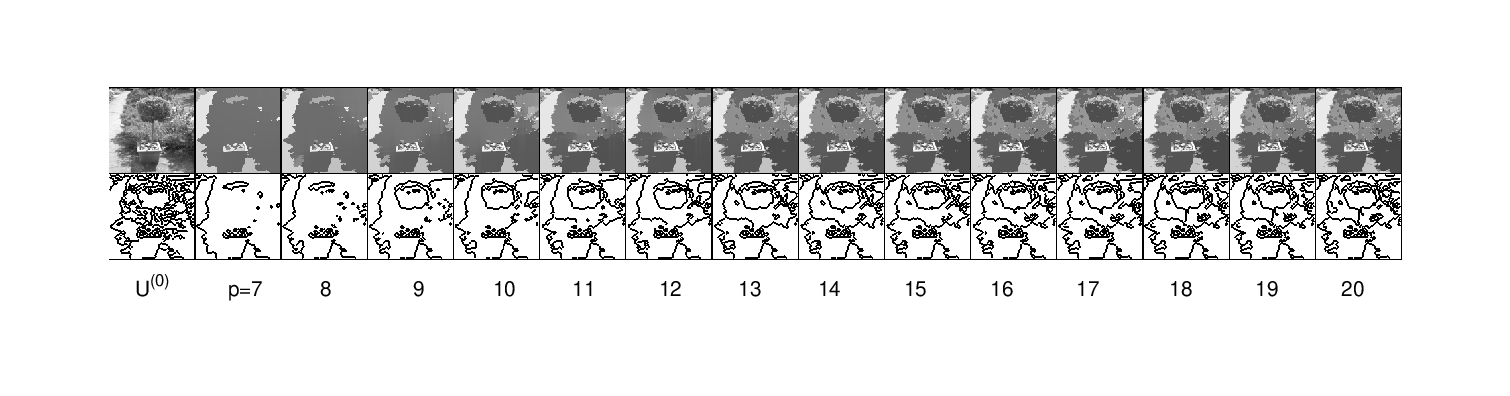}
\caption{\label{fig:BushVariosp} For $p=7,\, 8, \ldots, 20$ we show the original bush image ${\boldsymbol U^{(0)}}$ and ${\boldsymbol U^{(n)}}$ which corresponds to the last iteration performed. The image size was $128 \times 128$ and the number of iterations, $n$, was obtained from \eqref{eq:tiempo}. The time step was $k=200$ and the iterations performed were $n=51$. The edges were detected using Canny detector $(\sigma=1)$.}
\end{figure}

\subsubsection{Proposal tuning parameter $p$ by using prior knowledge}

To select the value of $p$, we have used an edge detector and a training image database with their manual segmentation, as ground truth. Fig.~\ref{fig:GroundT} shows some manual segmentations used for tuning the value of $p$.
Then we have considered
the \emph{precision} and \emph{recall} measures to calculate the \emph{ F-measure} \cite{van1979theoretical2,RefWorks:195}. Let $E_I$  be the set of edge pixels for the manual segmentation and $E_F^p$ the set of edge pixels of the filtered image,
$$P(p)= |E_I \cap E_F^p|/|E_F^p| \qquad R(p)=|E_I \cap E_F^p|/|E_I|$$
where $| \cdot |$ is the counting measure. That is, the  {\em precision} $P(\cdot)$, is the ratio of the true edge pixels detected to the total number of edge pixels detected, and the {\em recall} $R(\cdot)$ is the ratio of the true edge pixels detected to the number of edge pixels of the manual segmentation.

The \emph{F-measure} is defined as the weighted harmonic mean of precision and recall,
$$F(p,\alpha) = P(p)R(p)/(\alpha R(p) + (1-\alpha)P(p)) \;\; \alpha \in [0,1]$$
We chose: $p^{*}= \mbox{arg max} \{F(p,1/2): p > 1 \}$. In the experiments, we have used $p= 1.5, 2, 2.5, \ldots,  20$ and $\alpha= 1/2$ to attach equal importance to precision and recall.
\begin{figure}[H]
\centering
\includegraphics[width=0.9\linewidth,keepaspectratio]{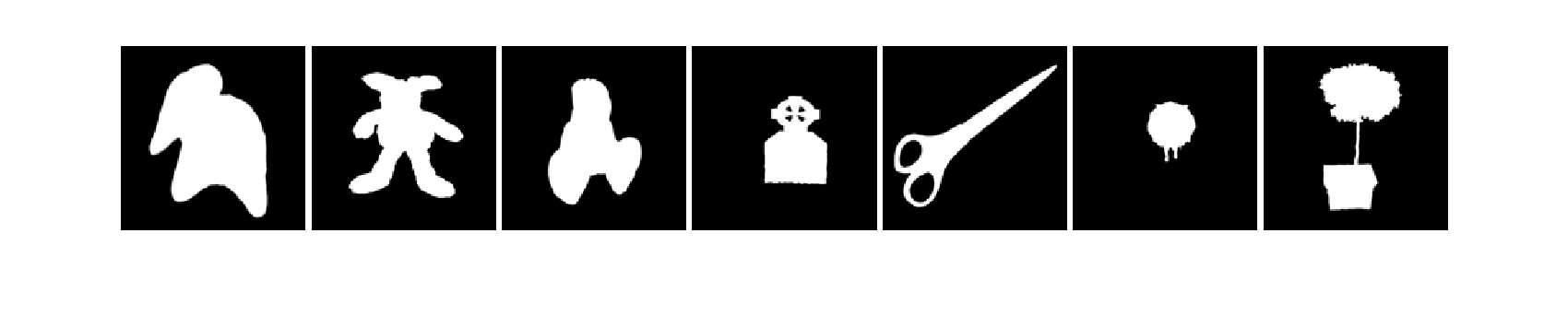}
\caption{\label{fig:GroundT} Ground truth images.}
\end{figure}

\begin{figure}[H]
\centering
\includegraphics[width=0.7\linewidth,keepaspectratio]{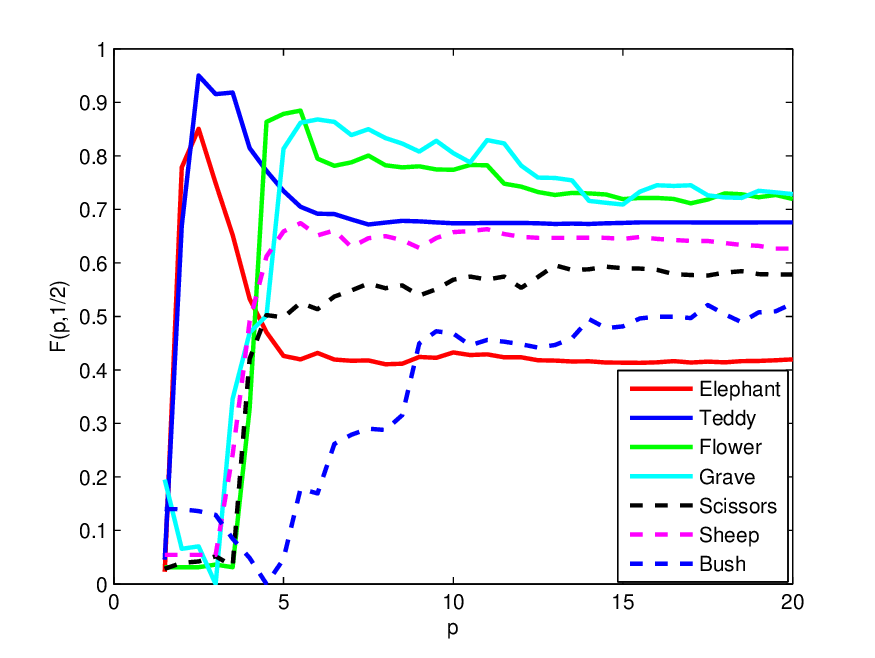}
\caption{\label{fig:MediaArmonica128k200} F-measure, $F(p,1/2)$, is shown for $p = 1.5, \, 2, \, 2.5, \ldots ,20$.}
\end{figure}

\begin{figure}[H]
\centering
\includegraphics[width=\linewidth,keepaspectratio]{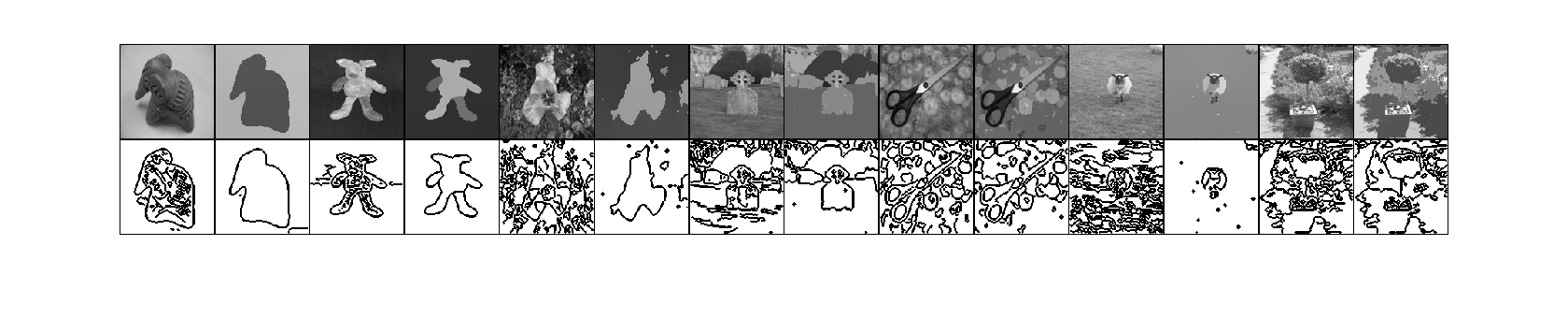}
\caption{\label{fig:todas128k200} The original image ${\boldsymbol U^{(0)}}$ and the image corresponding to the last iteration ${\boldsymbol U^{(n)}}$ are shown here. The size of all images is $128 \times 128$ and $k=200$. The $p^{*}$ values and the iterations performed are: elephant $p^{*}=2.5$, $n=39$, bear $p^{*}=2.5$, $n=39$, flower $p^{*}=5.5$, $n=42$, grave $p^{*}=6$, $n=35$, scissors $p^{*}=13$, $n=41$, sheep $p^{*}=5.5$, $n=32$ and bush $p^{*}=20$, $n=51$. The edges were detected using Canny detector $(\sigma=1)$.}
\end{figure}

\begin{table}[H]
\caption{ The first row shows, $F_0$, the value of the F-measure for the original images. The second row shows $p^{*}= \mbox{arg max} \{F(p,1/2): p=1.5,2, \ldots, 20\}$. The third row shows the maximum value of the F-measure for $p^{*}$.}\label{TablaF}
\begin{center}
\begin{tabular}{|l|lllllll|} \hline
\multicolumn{8}{|c|}{{\bf Images}}  \\ \hline
 & Bear & Bush & Elephant &  Flower & Grave & Scissors & Sheep   \\ \hline
$F_0$ & 0.597 & 0.578 & 0.379 &  0.573 & 0.606 &0.704 & 0.623  \\ \hline
$p^{*}$ & 2.5 & 20 & 2.5 & 5.5 & 6& 13 & 5.5   \\ \hline
$F(p^{*},1/2)$ & 0.918 & 0.554 & 0.852  & 0.913 & 0.832 & 0.693 & 0.699  \\ \hline
\end{tabular}
\end{center}
\end{table}

Fig.~\ref{fig:todas128k200} shows the performance of our filter with the proposed parameters. When it is applied to images with large and high contrast regions (elephant, bear, flower and grave) we see in Fig.~\ref{fig:MediaArmonica128k200} that the F-measure achieves an absolute maximum, edges are well preserved, no blurring is introduced and piecewise quasi-constant images are achieved. When the filter is applied to images in which the contrast between the object and the background is low or images with small objects and many local details along with edges that have low gradient (scissors, bush), the F-measure reaches a maximum value for a parameter $p$ high, in order to avoid blurring of the edges, but if we compare the maximum value for these cases there is not a significant improvement with the value of F-measure for the initial image, see Tab.~\ref{TablaF}. In such cases in order to improve the outcomes the use of local $\gamma$ has to be carefully analyzed.

%
%
Fig.~\ref{fig:comparative_filters} and Table~\ref{Table:comparative_filters} illustrate a comparative with two filter for obtaining piecewise constant images: (a) the cartoon variation image decomposition~\cite{RefWorks:165} and (b) a nonlinear diffusion filter which uses the AOS scheme ~\cite{RefWorks:42}. These approaches are chosen because their tasks are the processing of piecewise constant images and these filters have very few parameters to adjust. The first one only requires a weighting between the fidelity term and TV regularization. In the second approach, the diffusivity function is replaced in the proposal and the numerical scheme is the same. Experimental results show: (a) the stopping time procedure is robust. The diffusivity function is changed in the numerical scheme and the piecewise constant images are obtained. (b) The proposed approach gives better qualitative results. High contrast details are kept unlike other two approaches and (c) the F-measures of our approach are higher than in the other two filters.
The scripts used in this study are available at 
\url{https://github.com/cplatero/NonlinearDiffusion}.

\begin{figure}[H]
\centering
\includegraphics[width=\linewidth]{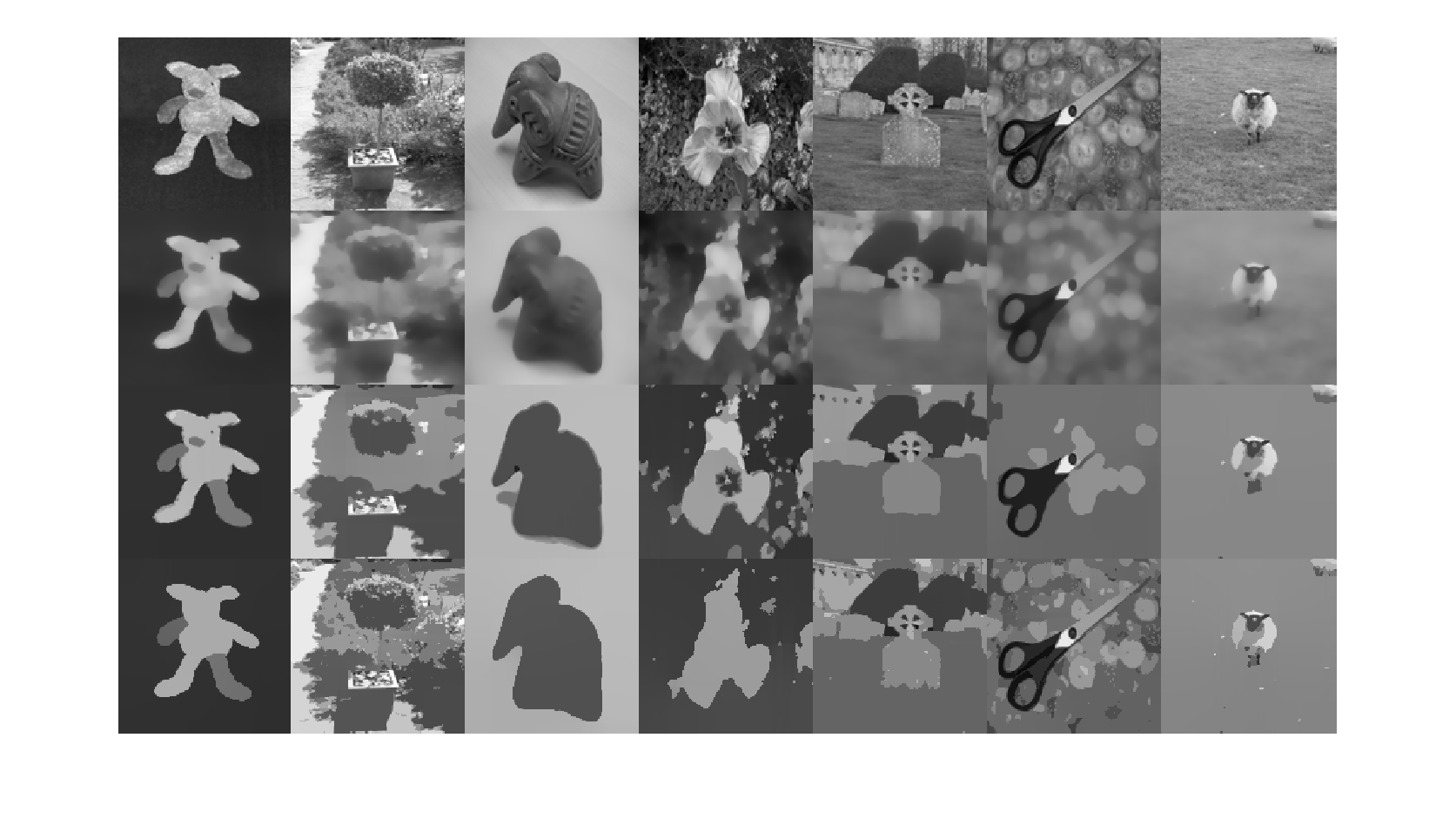}\\[4pt]
\includegraphics[width=\linewidth]{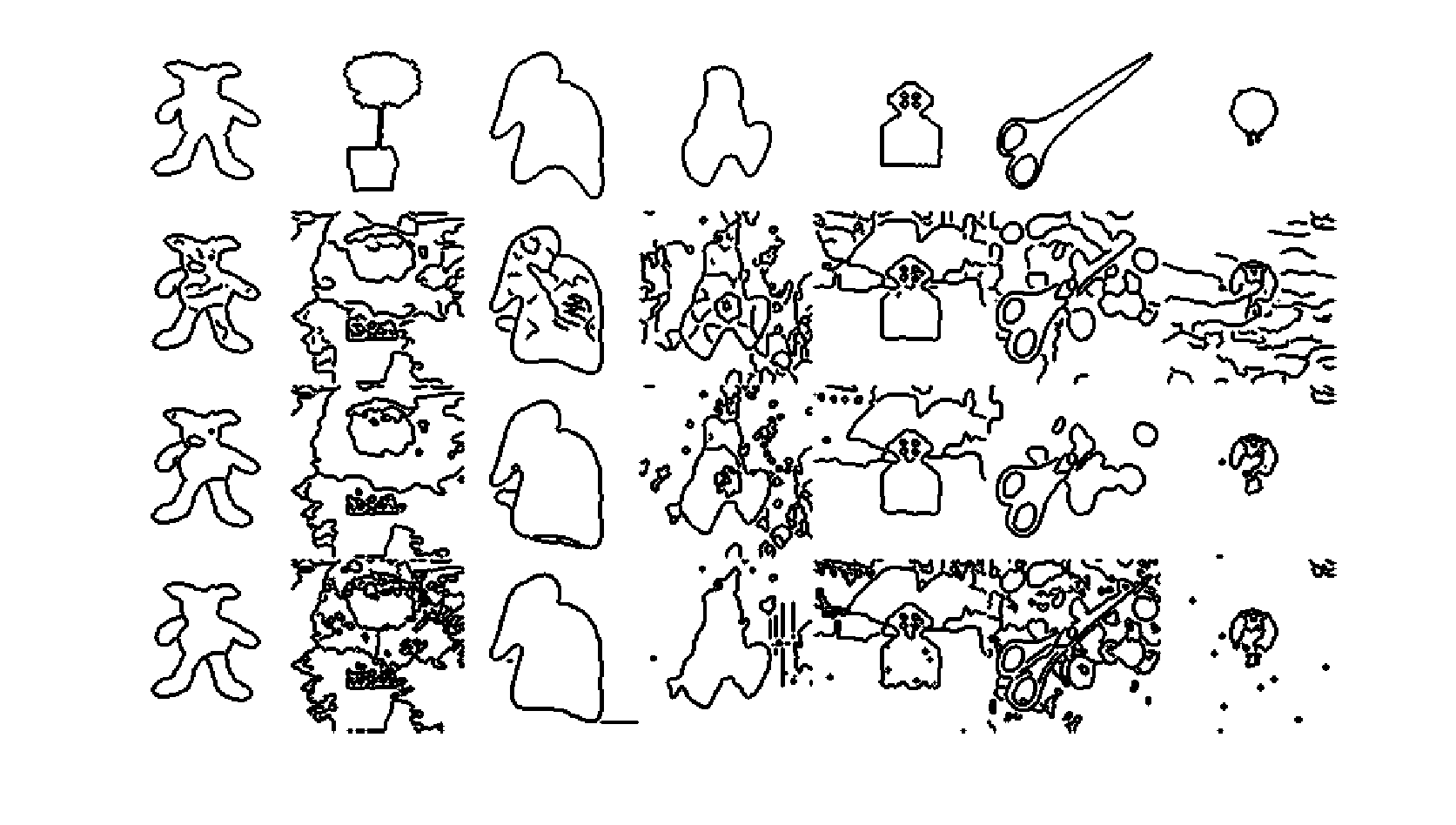}
\caption{The first block shows the filtered images and the second block illustrates their edge detections. Within each block, the first row shows the original images. The next rows illustrate the output of the image filtering with the cartoon variation image decomposition, the nonlinear diffusion filter and our proposal, respectively.}
\label{fig:comparative_filters}
\end{figure}

\begin{table}[H]
\caption{ A quantitative comparative of the image filtering with three approaches are reported using the F-measures. The first row shows the F-measure in the original images. The next rows illustrate the the F-measure in the cartoon variation image decomposition, the nonlinear diffusion filter and our proposal, respectively.}
\label{Table:comparative_filters}
\begin{center}
\begin{tabular}{|l|lllllll|} \hline
\multicolumn{8}{|c|}{{\bf Images}}  \\ \hline
 & Bear & Bush & Elephant & Flower &  Grave & Scissors & Sheep  \\ \hline
$F_0$ & 0.597 & 0.578 & 0.397 & 0.573 &0.606 & 0.704 & 0.623  \\ \hline
$F_{TV}$ & 0.793 & 0.472 & 0.538 & 0.621 &0.858 & 0.633 & 0.621  \\ \hline
$F_{DIL}$ & 0.888 & 0.460 & 0.808 & 0.699 &0.891 & 0.662 & 0.607  \\ \hline
$F(p^{*},1/2)$ & 0.918 & 0.554 & 0.852 & 0.913 & 0.832 & 0.693 & 0.699 \\ \hline
\end{tabular}
\end{center}
\end{table}



In this regard,
the proposed diffusivity was also tes\-ted with abdominal CT scans.
Obviously, the image domains were in 3D and according  to clinical protocols, the grids were not regular. We observed that the choice of $\gamma$ as global parameter was not a good one. Specifically, the weak edges were blurred and then some regions of interest were merged. To overcome these drawbacks, several strategies about a local $\gamma$ are  being studied. Our approach is to detect weak edges and sequently apply a local $\gamma$ that tends to zero for these areas. In the rest of the image, $\gamma$ could be global and calculated using $MAD$ function as we stated above. These experiments have validated the 3D extension of our filter, even with high values of $k$. Fig.~\ref{fig:liver} shows first preliminary results.
\begin{figure}[H]
\centering
\includegraphics[width=\linewidth, keepaspectratio]{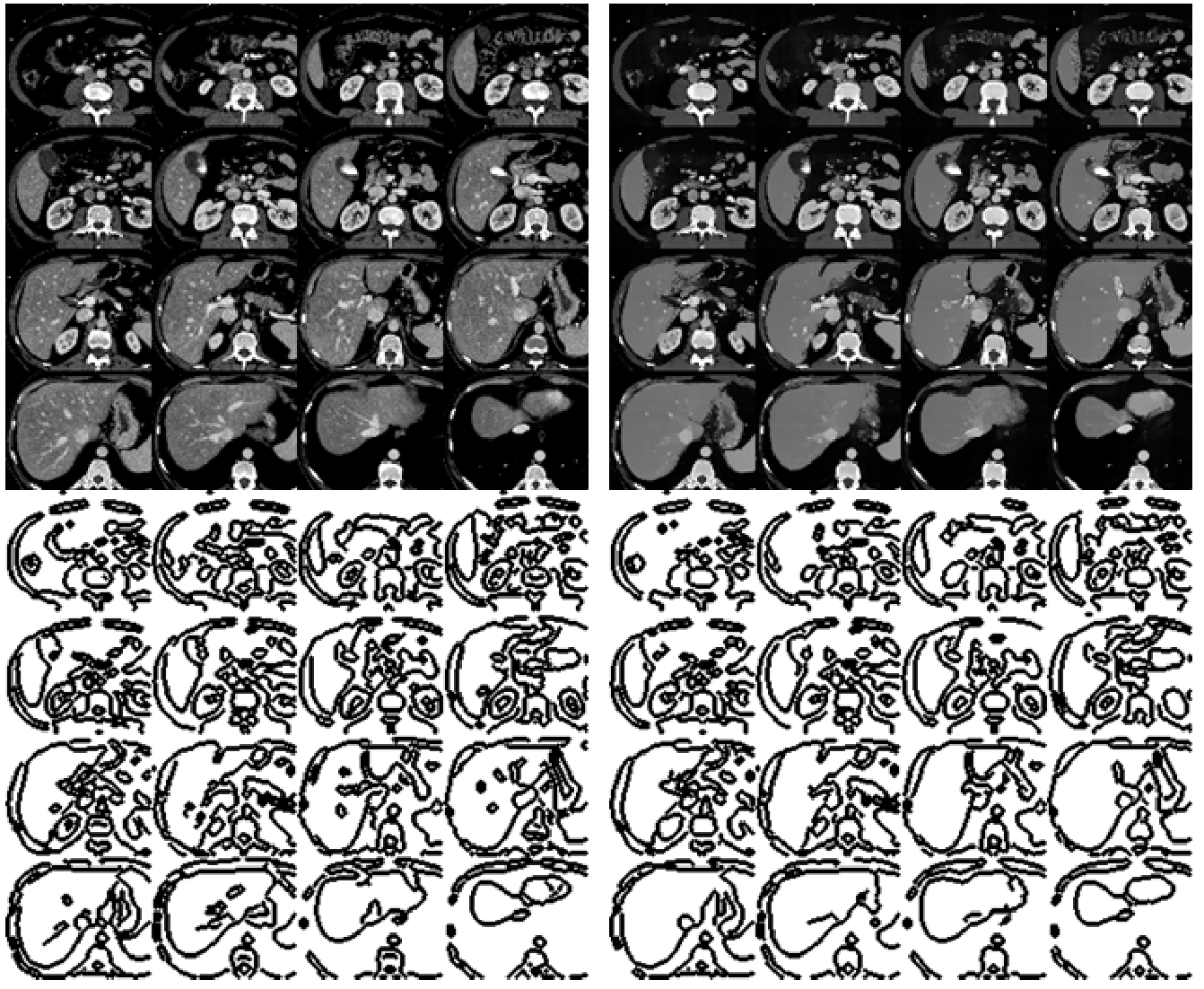}
\caption{\label{fig:liver} Some slices of the original CT and the corresponding filtered slices in the last iteration. The size of the image is $296\times357\times67$, pixel spacing is $0.76$ \emph{mm} and slice distance is $3$ \emph{mm}. The values are $p^{*}=10$, $k=8000$ and $n=15$.}
\end{figure}

\clearpage
\section{Discussion and Conclusion}
In the first part of this article, the intrinsic formulation for nonlinear diffusion equation was developed~\eqref{eq:intri_11}. This allowed us, firstly, to observe clearly the effect of the derivative of the diffusion function in the condition for the edge enhancement, and secondly to establish a theoretical framework for the design of the diffusivity function.
Because of this, diffusivity function~\eqref{eq:4} was proposed, capable of separating the initial image in piecewise constant regions using nonlinear diffusion techniques.

From the numerical point of view, it has been observed that with the proposed diffusivity~\eqref{eq:4}, the Newton's method was unstable, due to the matrix of iteration. Therefore, we have suggested to use
the tangential stiffness method given by~\eqref{eq:aprox_1}, which leads to the \emph{Picard} or \emph{direct} me\-thod \eqref{eq:iterativo}. In the case of the first iteration we obtain the \emph{semi-implicit method}. Thus, it is shown that with this numerical method we really perform forward diffusion instead of backward diffusion connecting with an edge-preserving process.

The diffusivity~\eqref{eq:4} has two parameters: $\gamma$ is related to the slopes at the edges of the regions, while $p$ is related with the speed of the forward diffusion. It also has interesting properties.
First, the value \-of $g\left(\|\nabla u\|^{2}\right) = 1$ for $\|\nabla u\|<\sqrt{\gamma}$ allows to use the \emph{setting time} $t_{s}$ as a stopping time.
Second, we have decided to set the parameter $\gamma$ by using the $MAD$ function in the different test images. So,
depending on the features of each image and using a set of training images, we can adjust the parameter $p$ as it was stated in the experimental part. This fact allows us to introduce prior knowledge for getting a piecewise constant image using the criterion: $p^{*}= \mbox{arg max} \{F(p,1/2): p >1 \}$. Thus, we get good results for images that have some contrast among the regions (see images in Fig.~\ref{fig:todas128k200}: bear, elephant, flower, grave, sheep).
Third, we emphasize that we only need to calculate once the \emph{setting time} $t_{s}$ for one image of the set and it can be used for the rest of the images if they have the same size and when $\vec{U}^{0}$ is similar. Fourth, the robustness of our diffusivity~\eqref{eq:4} leads  to the use of a large time step $k\equiv \Delta t$, therefore reducing the number of steps, as it can be observed in the liver image example (Fig~\ref{fig:liver}). This feature is an advantage over the functions of the type appearing in equation~\eqref{eq:condition_4}, for which it is necessary to use increments of time too small.
From the last two properties, it follows that the proposed diffusivity provides a nonlinear diffusion process for getting piecewise constant images with a low computational effort. This fact, combined with the no necessity to consider functional spaces to select the texture and the geometry of the image processing, give us an alternative over variational methods.
Finally, as previously mentioned, the di\-ffu\-si\-vi\-ty fun\-cti\-on was tested for 2D and 3D, with promising results.



It is worth noting that for images where the contrast is lower among different regions, the result is not conclusive (see above images in Fig.~\ref{fig:todas128k200}: scissors, bush). To improve outcomes in these cases we propose, as future work, to use a local value of $\gamma$. Fig.~\ref{fig:liver} shows some initial results of this approach.
Anyway,
these topics are not discussed in this article because its development and approach make the size of this item excessive. These approaches are  also postponed for future work.

Finally, a comparative study with other diffusivities \cite{RefWorks:30,RefWorks:157,chao2010improved,chao2010anisotropic} is currently under investigation.

\end{document}